\documentclass[11pt,a4paper]{article}
\usepackage{acl}
\usepackage{times}
\usepackage{latexsym}
\usepackage{amsthm}

\newtheorem{proposition}{Proposition}

\usepackage{amsmath}
\usepackage{graphicx}
\usepackage{algorithm}
\usepackage[noend]{algpseudocode}

\usepackage{multirow}
\usepackage{adjustbox}
\usepackage{bbm}
\usepackage{wrapfig,lipsum}
\usepackage{xspace}
\usepackage{booktabs}       
\usepackage[normalem]{ulem}
\usepackage{makecell}
\usepackage{amssymb}
\usepackage{pifont}
\newcommand{\cmark}{\ding{51}}%
\newcommand{\xmark}{\ding{55}}%

\usepackage{microtype}

\usepackage{cleveref}
\crefformat{section}{\S#2#1#3}
\crefformat{subsection}{\S#2#1#3}
\crefformat{subsubsection}{\S#2#1#3}
\crefrangeformat{section}{\S\S#3#1#4 to~#5#2#6}
\crefmultiformat{section}{\S\S#2#1#3}{ and~#2#1#3}{, #2#1#3}{ and~#2#1#3}
\usepackage{refstyle}
\Crefformat{figure}{#2Fig.~#1#3}
\Crefmultiformat{figure}{Figs.~#2#1#3}{ and~#2#1#3}{, #2#1#3}{ and~#2#1#3}
\Crefformat{table}{#2Tab.~#1#3}
\Crefmultiformat{table}{Tabs.~#2#1#3}{ and~#2#1#3}{, #2#1#3}{ and~#2#1#3}
\Crefformat{appendix}{Appx.~\S#2#1#3}
\crefformat{algorithm}{Alg.~#2#1#3}
\crefformat{equation}{Eq.~#2#1#3}
\crefformat{definition}{Def.~#2#1#3}

\newcommand{\stitle}[1]{\vspace{1ex} \noindent{\bf #1.}}

\newcommand{\modelname}{\textsc{GraphCache}\xspace}

\title{\modelname: Message Passing as Caching for Sentence-Level \\Relation Extraction}

\author{
Yiwei Wang$^1$ \ \ \ \ Muhao Chen$^2$ \ \ \ \ Wenxuan Zhou $^2$ \ \ \ \ Yujun Cai$^3$ \\ \textbf{Yuxuan Liang$^1$ \ \ \ \ Bryan Hooi$^1$} \\ 
$^1$ National University of Singapore \quad
$^2$ University of Southern California  \\
$^3$ Nanyang Technological University \\
\texttt{wangyw\_seu@foxmail.com}
}

\date{}

\begin{document}
\maketitle
\begin{abstract}
\textbf{Entity types} and \textbf{textual context} are essential properties for sentence-level relation extraction (RE).
Existing work only encodes these properties within individual instances, which limits the performance of RE given the insufficient features in a single sentence.
In contrast, we model these properties from the whole dataset and use the dataset-level information to enrich the semantics of every instance.
We propose the \modelname (\textbf{Graph Neural Network as Caching}) module, that propagates the features across sentences to learn better representations for RE.
\modelname aggregates the features from sentences in the whole dataset to learn \textbf{global} representations of properties, and use them to augment the \textbf{local} features within individual sentences.
The global property features act as dataset-level prior knowledge for RE, and a complement to the sentence-level features.
Inspired by the classical caching technique in computer systems, we develop \modelname to update the property representations in an online manner.
Overall, \modelname yields significant effectiveness gains on RE and enables efficient message passing across all sentences in the dataset.
\end{abstract}

\section{Introduction}

Sentence-level relation extraction (RE) aims at identifying the relationship between two entities mentioned in a sentence.
RE is crucial to the structural perception of human language, and also benefits many NLP applications such as automated knowledge base construction \cite{distiawan2019neural}, event understanding \cite{wang2020joint}, discourse understanding \cite{yu2020dialogue}, and question answering \cite{zhao2020condition}.
The modern tools of choice for RE are the large-scale pretrained language models (PLMs) that are used to encode individual sentences, therefore obtaining the sentence-level representations \cite{liu2019roberta,joshi-etal-2020-spanbert,yamada-etal-2020-luke}.

Existing work considers \textbf{entity types} and \textbf{textual context} as essential \textbf{properties} for RE \cite{peng2020learning,peters2019knowledge,zhou2021improved}.
Nonetheless, most existing RE models only capture these properties \textit{locally} within individual instances, while not \textit{globally} modeling them from the whole dataset.
Given the insufficient features of a single sentence, it is beneficial to model these properties from the whole dataset and use them to enrich the semantics of individual instances.

To overcome the aforementioned limitation, we propose to mine the entity and contextual information beyond individual instances so as to further improve the relation representations.
Particularly, we first construct a heterogeneous graph to connect the instances sharing common properties for RE.
This graph includes the sentences and \textit{property caches}.
Each cache represents a property of entity types or contextual topics.
We connect every sentence to the corresponding property caches (see \Cref{fig:1}), and perform message passing over edges based on a graph neural network (GNN).
In this way, the property caches aggregate the features from connected sentences, which will act as a complement to the sentence-level features and provide prior knowledge when identifying relations. 

The constructed graph connecting sentences has the same scale as the whole dataset, which leads to high computational complexity of the GNN.
To address this issue, our idea is to view the message passing of GNNs as data loading in computer systems, adapting the classical caching techniques to efficiently mining the property information from all sentences.
We encapsulate this computational idea in a new GNN module, called \modelname (\textbf{Graph Neural Network as Caching}), that uses an online updating strategy to refresh the property caches' representations.
In addition, we design an attention-based global-local fusion module to augment the sentence-level representations using the property caches with adaptive weights.

\modelname can be incorporated into popular RE models to improve their effectiveness without increasing their time complexity, as analyzed in theory (\Cref{sec:3_2}). 
As far as we know, ours is the first work to propagate the features across instances to enrich the semantics for sentence-level RE.
We evaluate \modelname on three public RE benchmarks including TACRED \cite{zhang2017tacred}, SemEval-2010 task 8 \cite{hendrickx2019semeval}, and  TACREV \cite{alt-etal-2020-tacred}.
Empirical results show that \modelname consistently improves the effectiveness of popular RE models by a significant margin and propagates features between all sentences in an efficient manner.

\begin{figure}[!tb]
	\centering
	\includegraphics[width=1\linewidth]{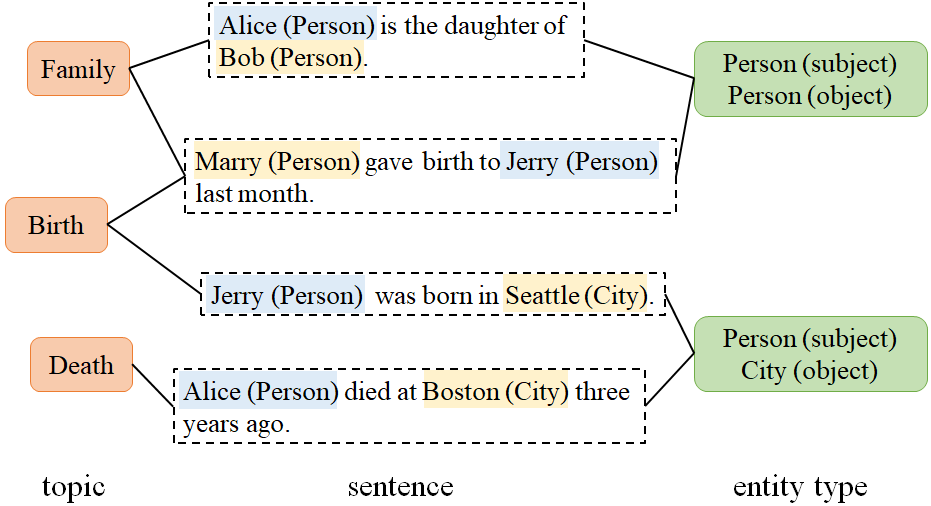}
	\caption{
		We construct a heterogeneous graph to connect the sentences sharing common properties for RE.
		We consider two kinds of properties: contextual topics and entity types.
		\label{fig:1}}
\end{figure}


\section{Related Work}
\stitle{Sentence-Level Relation Extraction}
Early research efforts~\cite{zeng-etal-2014-relation,wang-etal-2016-relation,zhang2017tacred} train RE models from scratch based on lexicon-level features.
Recent work has shifted to fine-tuning pretrained language models (PLMs; \citealt{devlin-etal-2019-bert,liu2019roberta}) resulting in better performance.
For example,
BERT-MTB \cite{baldini-soares-etal-2019-matching} continually finetunes the PLM with a matching-the-blanks objective that decides whether two sentences share the same entity.
SpanBERT \cite{joshi-etal-2020-spanbert} pretrains a masked language model on random contiguous spans to learn span-boundaries and predict the entire masked span.
LUKE \cite{yamada-etal-2020-luke} extends the PLM's vocabulary with entities from Wikipedia and proposes an entity-aware self-attention mechanism.
K-Adapter~\cite{wang2020k} fixes the parameters of the PLM and uses feature adapters to infuse factual and linguistic knowledge.
Despite their effectiveness, most existing work on sentence-level RE exploits the entity information and context within only an individual instance, while we propose to globally capture the semantic information from the whole dataset to augment the relation representations.
Our model can be flexibly plugged into existing RE models and improve their effectiveness without increasing the time complexity.

\stitle{Graph Neural Networks for Natural Language Processing}
Due to the large body of work on applying GNNs to NLP, we refer readers to a recent survey \cite{wu2021graph} for a general review.
GNNs have been explored in several NLP tasks such as semantic role labeling \cite{marcheggiani2017encoding}, machine translation \cite{bastings2017graph}, and text classification \cite{henaff2015deep,defferrard2016convolutional,kipf2016semi,peng2018large,yao2019graph}.
GNNs have also been widely adopted in various variants of relation extraction on the sentence level, \cite{zhang2018graph,zhu2019graph,guo-etal-2019-attention}, the document level \cite{sahu2019inter,christopoulou2019connecting,nan-etal-2020-reasoning,zeng2020double}, and the dialogue level \cite{xue2021gdpnet}.
However, on the sentence-level relation extraction, most existing work \cite{zhang2018graph,guo2019attention,wu2019simplifying} uses the graph neural networks to encode the relation representations from individual instances instead of operating the message passing between instances.
In contrast, we build a heterogeneous graph to connect the instances that share the properties for RE, and design the caching updater to efficiently perform the message passing between instances.

\section{Methodology}

\stitle{Task Definition} 
Sentence-level relation extraction (RE) aims to identify the relation between a pair of entities in a sentence.
In this task, each instance is composed of a sentence, the subject and object entities, and entity types.
For example, in the sentence \textit{`\uwave{Mary} gave birth to \underline{Jerry} at the age of 21.'}\footnote{We use \underline{underline} and \uwave{wavy line} to denote subject and object respectively by default.}, \textit{`Mary'} and \textit{`Jerry'} are the entities, the entity types are both \textit{person}, and the ground-truth relation between \textit{`Jerry'} and \textit{`Mary'} is \textit{parent}.

We propose \modelname (Graph Neural Networks as Caching) as a message passing methodology to model the dataset-level property representations and use them to enrich every instance's semantics.
\modelname creates a graph representation where sentences with shared property information are connected with property caches.
\modelname first models the global semantic information by aggregating the features from the whole dataset, and then fuses the global and local features to augment the relational representations for every sentence.

We analogize the message passing in GNNs to caching in computer systems. 
Caching is about loading data from high volume disks to low volume caches, so as to accelerate data loading.
Analogously, when GNNs perform the message passing between sentences through a smaller number of bridge nodes, we can think of the massive sentences in the dataset as the disk data, and the properties, which aggregates the features from sentences, as caches.
\modelname can be flexibly plugged into existing RE models.
As far as we know, ours is the first work to propagate the features between instances to enrich the semantics for RE.
\modelname takes an existing RE model as the backbone, e.g., BERT, and takes the sentence-level  representations given by the backbone as the inputs of message passing.

A \modelname module consists of three key components:
(i) \textit{A }\textit{graph construction technique} builds a few property caches. Each cache represents a property for RE: entity type or contextual topic.
We connect each sentence to its corresponding properties, so that every property aggregates the features from its neighbor sentences.
(ii) \textit{Caching message passing} aggregates the sentence-level representations to model the properties' representations in an online manner.
(iii) \textit{Global-local fusion} fuses the global property representations and local sentence-level ones to augment the relation representations.
Next, we will discuss the three main components in more detail.

\subsection{Graph Construction for Sentence-level Relation Extraction}\label{sec:3_1}

We build a large and heterogeneous graph to connect the sentences sharing the properties: \textbf{entity types} and \textbf{textual context}, which are essential for RE \cite{peng2020learning,peters2019knowledge,zhou2021improved}.
The heterogeneous graph is defined as $G = (\mathcal{V}, \mathcal{E})$, where $\mathcal{V}$ is the set of nodes, and $\mathcal{E}$ is the set of edges. 
$\mathcal{V} = \mathcal{V}_S \cup \mathcal{V}_P$, where $\mathcal{V}_S$ is the set of sentences, and $\mathcal{V}_P = \mathcal{V}_C \cup \mathcal{V}_E$ is the property caches. 
Here $\mathcal{V}_C$ is the set of latent topics \cite{zeng2018topic} 
mined from the latent topics from the text corpus using LDA \cite{blei2003latent}, which has been found effective in modeling useful contextual patterns \cite{jelodar2019latent}.
Each topic is represented by a probability distribution over the words, and we assign each sentence to the top $P$ topics with the largest probabilities.
$\mathcal{V}_E$ is the set of entity types, where every cache represents the types of an entity pair.
The entity types are also crucial for predicting relations \cite{peng2020learning,zhou2021improved}. 
An edge $(p,s)\in\mathcal{E}$ exists if the sentence $s\in \mathcal{V}_S$ has the property $p\in \mathcal{V}_P$.

We will implement a GNN on this graph.
Specifically, to incorporate the global property information into relation extraction, the property caches aggregates the features from the connected neighboring sentences. 
This step enables property caches to globally model the properties from the whole dataset.
We then use the global property representations from the caches to enrich every sentence's semantics.
In this way, the property caches act as prior knowledge when identifying relations and provide each sentence with more representative features.

\begin{figure*}[!tb]
	\centering
	\includegraphics[width=1\linewidth]{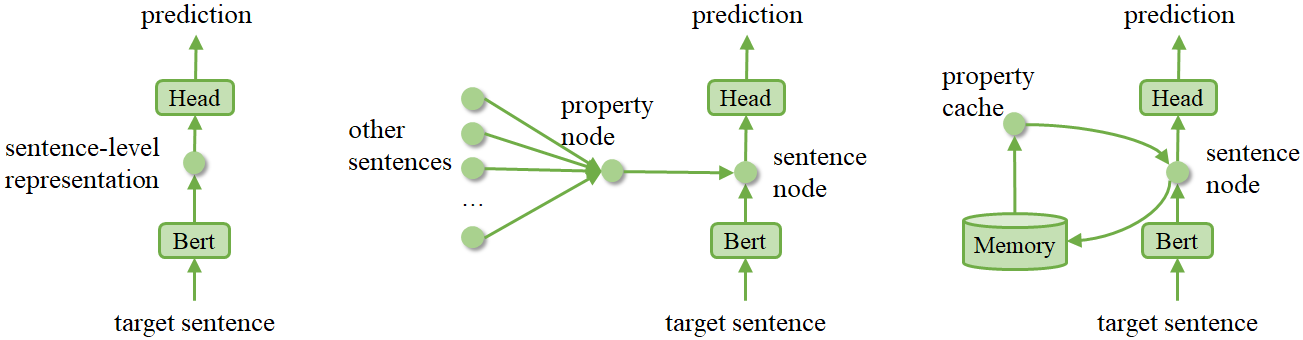}
	\caption{
		(left) Existing models encode individual instances for RE.
		(middle) In standard GNNs \cite{kipf2016semi}, we predict for an instance by aggregating the features from many other sentences in the dataset, leading to high time complexity.
		(right) Our \modelname implements a caching updater (see \Cref{eq:update}) to update the properties' representations in an online manner, which significantly reduces the time complexity.
		\label{fig:2}}
\end{figure*}

\subsection{Caching Message Passing} \label{sec:3_2}
We take an existing RE model as the backbone, e.g., BERT \cite{devlin-etal-2019-bert}, which produces the sentence-level representation as $\mathbf{h}_s$.
Next, we deploy a two-layer GNN on our heterogeneous graph for message passing across sentences. 
Specifically, the first GNN layer aggregates the sentence-level representations to property caches at the $t$th training step:
\begin{align}\label{eq:gnn1}
\bar{\mathbf{h}}_p(t) & = \mathrm{MEAN}\left(\{\mathbf{h}_{s}(t), s \in \mathcal{N}(p)\}\right), \nonumber
\\
\mathbf{h}_p(t) & = \mathrm{FFN}\left(\bar{\mathbf{h}}_p(t)\right),
\end{align}
where $p\in \mathcal{V}_P$ is a property, $s\in \mathcal{N}(p)$ is a sentence having property $p$, $\mathrm{MEAN}(\cdot)$ is the mean aggregator \cite{hamilton2017inductive}, and $\mathrm{FFN}(\cdot)$ is the feed-forward network.
$\mathrm{FFN}(\cdot)$ can be a linear layer in SGC \cite{wu2019simplifying}, a linear layer followed by a nonlinear activation function in GraphSAGE \cite{hamilton2017inductive}, or a multi-layer perception in GIN \cite{xu2018powerful}, etc.
We follow SGC \cite{wu2019simplifying} to implement $\mathrm{FFN}(\cdot)$ by default. 
For each property $p$, this layer aggregates the sentence-level representations $\mathbf{h}_s(t)$ from $s \in \mathcal{N}(p)$ to obtain a global property embedding $\mathbf{h}_p(t)$. 
In this way, the generalized context of each property is captured from the whole dataset, which is further used to enhance the relation representations for each sentence in the second GNN layer. 
We describe the details of the second GNN layer in \Cref{sec:second-GNN}.

\begin{algorithm}[!t]
	\caption{\modelname for Relation Extraction}\label{alg:3}
	\begin{flushleft}
		\textbf{Input:} The number of training steps $T$, the dataset $\mathcal{D} = \{se_s, r_s| s = 1, 2, \dots, N\}$, where $se_s, r_s$ are the sentence and relation of the $s$th instance, our graph $G$ defined in \Cref{sec:3_1}, and the batch size $B$.
		\\
		\textbf{Output:} The model's trained parameters.
	\end{flushleft}
	\begin{algorithmic}[1]
	    \State Initialize the model's parameters as random values, and initialize the values of memory $\mathcal{M}$ and property caches $\hat{\mathbf{h}}_p(t)$ as zero.
		\For{$t$ $\leftarrow$ 1 to $T$}
		\State Sample a batch $\mathcal{B}(t)$ from $\mathcal{D}$.
		\For{$s$ in $\mathcal{B}(t)$}
		\State $\mathbf{h}_s(t) \leftarrow \mathrm{Backbone}$($se_s$)
		\EndFor
		\State \textbf{end for}
		\For{$p$ in $\mathcal{V}_p$}
		\State Update $\hat{\mathbf{h}}_p(t)$ as \Cref{eq:update}.
		\State $\mathbf{h}_p(t) \leftarrow \mathrm{FFN}(\hat{\mathbf{h}}_p(t))$ as \Cref{eq:gnn1}.
		\EndFor
		\State \textbf{end for}
		\For{$s$ in $\mathcal{B}(t)$}
		\State Update $\hat{r}_s(t)$ as \Cref{eq:head}.
		\State $\mathcal{M}[s] \leftarrow \mathbf{h}_s(t)$.
		\EndFor
		\State \textbf{end for}
		\State Back-propagate to update the parameters by minimizing the cross entropy loss between $\hat{r}_s(t)$ and $r_i$ of instances in $\mathcal{B}$.
		\EndFor
		\State \textbf{end for}
	\end{algorithmic}
\end{algorithm}

Recall our heterogeneous graph for RE defined in \Cref{sec:3_1}. 
At each training step, classical GNNs perform message passing across edges between the sentences and properties.
In this case, the time complexity of the first GNN layer at each training step is $|\mathcal{E}|$. 
Note that $|\mathcal{E}|$ is larger than $|V_s|$, which is the number of sentences in the dataset. 
This leads to poor scalability of GNN, since $|V_s|$ is large in practice. 

To address this efficiency issue, we propose Caching GNN for RE in \Cref{alg:3}.
Our \modelname implements a memory dictionary $\mathcal{M}$ to store the sentence-level representations from the backbone.
To keep consistency with the updating parameters during training, we deploy a caching updater to refresh the properties' representations at each training step:

\begin{align}\label{eq:update}
&\hat{\mathbf{h}}_p(t) \nonumber\\ 
= & \mathrm{Updater}(\hat{\mathbf{h}}_p(t - 1) , \{\mathbf{h}_s(t), s \in \mathcal{B}(t)\})\nonumber\\
= & \hat{\mathbf{h}}_p(t - 1) + \sum_{s \in \mathcal{N}(p) \cap \mathcal{B}(t)} \frac{\mathbf{h}_{s}(t) - \mathcal{M}[s]}{|\mathcal{N}(p)|},
\end{align}
where $\mathcal{B}(t)$ denotes the batch at the $t$th training step. 
By doing so,  \modelname greatly reduces the time complexity from $|\mathcal{E}|$ to $|\mathcal{B}(t)|$ at each training step by using $\mathrm{Updater}$ to obtain the property caches' representations $\hat{\mathbf{h}}_p(t)$.

Our caching updater is much more efficient than the classical message passing of GNNs, since $|\mathcal{B}(t)| \ll |V_s| < |\mathcal{E}|$ generally holds in practice.
When we aggregate the sentence-level representations from $\mathcal{M}$, we provide the following proposition to show that our cache updater is as effective as the first GNN layer in \Cref{eq:gnn1}.

\begin{proposition}
	At the $t$th training step, denote the property caches' representations in the first GNN layer (see \Cref{eq:gnn1}) as $\bar{\mathbf{h}}_p(t)$, and those returned by our updater in \Cref{eq:update} as $\hat{\mathbf{h}}_p(t)$.
	There is $\hat{\mathbf{h}}_p(t) = \bar{\mathbf{h}}_p(t)$ for $\forall p\in \mathcal{V}_P, t > 0$.
\end{proposition}
\begin{proof}
When $t > 1$, if $\hat{\mathbf{h}}_p(t-1) = \bar{\mathbf{h}}_p(t-1)$, we have:
\begin{align}
&\hat{\mathbf{h}}_p(t) \nonumber\\ 
= & \mathrm{Updater}(\hat{\mathbf{h}}_p(t - 1) , \{\mathbf{h}_s(t), s \in \mathcal{B}(t)\})\\
= & \hat{\mathbf{h}}_p(t - 1) + \sum_{s \in \mathcal{N}(p) \cap \mathcal{B}(t)} \frac{\mathbf{h}_{s}(t) - \mathcal{M}[s]}{|\mathcal{N}(p)|} \nonumber\\
= & \sum_{s \in \mathcal{N}(p)} \frac{ \mathcal{M}[s]}{|\mathcal{N}(p)|} + \sum_{s \in \mathcal{N}(p) \cap \mathcal{B}(t)} \frac{\mathbf{h}_{s}(t) - \mathcal{M}[s]}{|\mathcal{N}(p)|} \\
= & \sum_{s \in \mathcal{N}(p) \backslash \mathcal{B}(t)} \frac{ \mathcal{M}[s]}{|\mathcal{N}(p)|} + \sum_{s \in \mathcal{N}(p) \cap \mathcal{B}(t)} \frac{\mathbf{h}_{s}(t)}{|\mathcal{N}(p)|} \\
=& \bar{\mathbf{h}}_p(t).
\end{align}
Besides, because $\hat{\mathbf{h}}_p(0) = \bar{\mathbf{h}}_p(0)$ for $\forall p\in \mathcal{V}_P$ holds as initialized in Alg. \Cref{alg:3}, we have
$\hat{\mathbf{h}}_p(t) = \bar{\mathbf{h}}_p(t)$ for $\forall p\in \mathcal{V}_P, t > 0$.
\end{proof}
\subsection{Global-Local Fusion} \label{sec:second-GNN}
In the second GNN layer, we propagate the properties' representations from the property cache to their neighboring sentences in the batch. 
Since a sentence $s$ may have more than one latent topic $|\mathcal{V}_C \cap \mathcal{N}(i)| > 1$, we utilize the attention mechanism to enable the target sentence to attend to different topics with adaptive weights.
\begin{align*}
\mathbf{h}_{s}^{topic}(t) = \mathrm{Attention}(\mathbf{h}_{s}(t), 
\{\mathbf{h}_{p}(t), p \in \mathcal{V}_C \}),
\end{align*}
where we follow \cite{vaswani2017attention} to implement $\mathrm{Attention}$.
The output $\mathbf{h}_{s}^{topic}(t)$ is the topic embedding fused for sentence $s$.
In this way, a sentence can be trained to attend to more relevant topics with higher weights.

Next, we have the entity type embedding of sentence $s$ as $\mathbf{h}_{s}^{entity}(t) = \mathbf{h}_{p}(t), p \in \mathcal{V}_E \cap \mathcal{N}(s)$, where $p \in  \mathcal{V}_E \cap \mathcal{N}(s)$ is the entity type node connected to sentence $s$.
$\mathbf{h}_{s}^{topic}(t)$ and $\mathbf{h}_{s}^{entity}(t)$ are the global representations of the properties related to sentence $s$, while $\mathbf{h}_s$ is the local representation of sentence $s$.
We fuse the global and local representations to enrich the semantics of sentence $s$ through a sentence-wise head:
\begin{align}\label{eq:head}
\hat{r}_i(t) = \mathrm{Head}\left(\mathbf{h}_{s}(t) \| \mathbf{h}_{s}^{topic}(t) \| \mathbf{h}_{s}^{entity}(t)\right),
\end{align}
where $\|$ denotes concatenation.
\modelname makes sentence-wise relation predictions $\hat{r}_i(t)$ using a sentence-wise Head, implemented as a multi-layer perception (MLP), analogous to a PointNet \cite{qi2017pointnet}.
Since \modelname predicts a relation label for each sentence, it can be trained by standard task-specific classification losses, e.g., cross-entropy \cite{mannor2005cross}.
During inference, we take $\hat{r}_i(t)$ after convergence as the output for RE.

\section{Experiments}

In this section, we evaluate the effectiveness of our \modelname method when incorporated into various RE models.
We compare our methods against a variety of strong baselines on the task of sentence-level RE.
We closely follow the experimental setting of the previous work \cite{zhang2017tacred,zhou2021improved,zhang2018graph} to ensure a fair comparison, as detailed below.

\subsection{Experimental Settings}
\stitle{Datasets}
We use the standard sentence-level RE datasets: TACRED \cite{zhang2017tacred}, SemEval-2010 Task 8 \cite{hendrickx2019semeval}, and TACREV \cite{alt2020tacred} for evaluation.
TACRED contains over 106k mention pairs drawn from the yearly TAC KBP challenge.
SemEval does not provide entity type annotations, for which we only construct the topic caches for message passing.
\citet{alt2020tacred} relabeled the development and test sets of TACRED to build TACREV.
The statistics of these datasets are shown in \Cref{tab:data}.
We follow \cite{zhang2017tacred} to use F1-micro as the evaluation metric.

\begin{table}[tb!]
	\centering
	
	\begin{adjustbox}{width=\linewidth}
		
		\begin{tabular}{@{}l| c c c c @{}}
			\toprule
			\textbf{Dataset}
			& $\#$\textbf{Train}	
			& $\#$\textbf{Dev}
			& $\#$\textbf{Test}
			& $\#$\textbf{Classes} \\ 
			\midrule
			\midrule
			TACRED & 68,124 & 22,631 & 15,509 & 42 \\
			SemEval & 6,507 & 1,493 & 2,717 & 19 \\
			TACREV & 68,124 & 22,631 & 15,509 & 42 \\
			\bottomrule
		\end{tabular}
		
	\end{adjustbox}
		\caption{Statistics of datasets.}\label{tab:data}
\end{table}

\begin{table}[tb!]
	\centering
	\setlength{\tabcolsep}{1pt}
	\begin{adjustbox}{width=\linewidth}
		\begin{tabular}{@{}l c c c @{}}
		   
			\toprule
			\textbf{Method}
			& \textbf{TACRED}	
			& \textbf{SemEval}
			& \textbf{TACREV}\\ 
			\midrule
			\midrule
			PA-LSTM \cite{zhang2017tacred} & 65.1 & 82.1 & 73.3 \\
			GCN \cite{zhang2018graph} & 64.0 & 80.7 & 71.9 \\
			C-GCN \cite{zhang2018graph} & 66.4 & 84.2 & 74.6 \\
			C-SGC \cite{wu2019simplifying} & 67.0 & 84.8 & 75.1 \\
			SpanBERT \cite{joshi-etal-2020-spanbert} & 70.8 & 86.1 & 78.0 \\
			RECENT \cite{lyu2021relation} & 75.2 & 85.8 & 83.0 \\
			IRE$_{\mathrm{BERT}}$ \cite{zhou2021improved} & 72.9 & 86.4 & 81.3 \\
			\midrule
			LUKE \cite{yamada-etal-2020-luke} & 72.7 & 87.8 & 80.6 \\
			LUKE + \modelname (ours) & 74.8 & \textbf{89.1} & 81.5 \\
			\midrule
			IRE$_{\mathrm{RoBERTa}}$ \cite{zhou2021improved} & 74.6 & 87.5 & 83.2 \\
			IRE$_{\mathrm{RoBERTa}}$ + \modelname (ours) & \textbf{75.5} & 88.2 & \textbf{84.2} \\
			\bottomrule
		\end{tabular}

	\end{adjustbox}
	\caption{F1 scores (\%) of Relation Extraction on the test set of TACRED, SemEval, and TACREV. The best results in each column are highlighted in \textbf{bold} font.\label{tab:1}}
\end{table}

\stitle{Compared Methods}
We compare \modelname with the following state-of-the-art RE models:
(1) \textbf{PA-LSTM}~\cite{zhang2017tacred} extends the bi-directional LSTM by incorporating positional information to the attention mechanism.
(2) \textbf{GCN}~\cite{zhang2018graph} uses a graph convolutional network to gather relevant contextual information along syntactic dependency paths.
(3) \textbf{C-GCN}~\cite{zhang2018graph} combines GCN and LSTM, leading to improved performance than each method alone.
(4) \textbf{C-SGC}~\cite{wu2019simplifying} simplifies GCN by removing the nonlinear layers and achieves higher effectiveness. 
(5) \textbf{SpanBERT}~\cite{joshi-etal-2020-spanbert} extends BERT by introducing a new pretraining objective of continuous span prediction.
(6) \textbf{RECENT} \cite{lyu2021relation} restricts the candidate relations based on the entity types.
(7) \textbf{LUKE}~\cite{yamada-etal-2020-luke} pretrains the language model on both large text corpora and knowledge graphs and further proposes an entity-aware self-attention mechanism. 
(8) \textbf{IRE}~\cite{zhou2021improved} proposes an improved entity representation technique in data preprocessing, which enables RoBERTa to achieve state-of-the-art performance on RE.

\stitle{Model Configuration}
For the hyper-parameters of the considered baseline methods, e.g., the batch size, the number of hidden units, the optimizer, and the learning rate, we set them as 
those in the original papers.
For 
LDA used in \modelname, we set the number of topics $K$ as 50, and the number of top relevant topics for every sentence $P$ as 2.
For all experiments, we report the median F-1 scores of five runs of training using different random seeds.

\subsection{Overall Performance}\label{sec:4_1}

We incorporate the \modelname framework with LUKE and IRE$_{\mathrm{RoBERTa}}$, and report the results in
\Cref{tab:1}.
Our \modelname method improves LUKE by 2.9\% on TACREV, 1.5\% on SemEval, and 1.1\% on TACREV in the F1 score.
For IRE$_{\mathrm{RoBERTa}}$, \modelname leads to the improvement of 1.2\% on TACRED, 0.8\% on SemEval, 1.2\% on Re-TACRED.
As a result, our \modelname achieves substantial improvements for LUKE and IRE$_{\mathrm{RoBERTa}}$ and enables them to outperform the baseline methods.

Note that LUKE and IRE$_{\mathrm{RoBERTa}}$ are both 
based on large pre-trained models, which 
have sufficiently large learning capacity to encode the individual instances.
In this case, our \modelname still improves their effectiveness by a large margin, which validates the benefits of modeling the properties: entity types and contextual topics, globally from the whole dataset.
This is due to the use of the global property representations that enrich the semantics of each instance, which effectively act as prior knowledge that helps identify the relations and complements the sentence-level features.

\subsection{Efficiency and Effectiveness of \modelname}
As analyzed in \Cref{sec:3_2}, \modelname enhances the backbone RE models without increasing their time complexity. 
In the experiments, we analyze the efficiency and effectiveness of \modelname on the TACRED dataset, following the experimental setting of RE in \Cref{sec:4_1}.

The methods we evaluate include IRE$_{\mathrm{RoBERTa}}$, IRE$_{\mathrm{RoBERTa}}$ implemented with classical GNN for message passing, and IRE$_{\mathrm{RoBERTa}}$ with our \modelname.
\Cref{tab:time} reports the performance, where `Time' is the training time until convergence using a Linux Server with an Intel(R) Xeon(R) E5-1650 v4 @ 3.60GHz CPU and a GeForce GTX 2080 GPU.

\begin{table}[tb!]
\small
	\centering
	
	\begin{adjustbox}{width=\linewidth}
		\begin{tabular}{@{}l|c|c|c@{}}
			\toprule
			\textbf{Method} & \textbf{Complexity} & \textbf{Time} & \textbf{F1 (\%)} \\ 
			\midrule
			\midrule
			IRE$_{\mathrm{RoBERTa}}$ \cite{zhou2021improved} & $\mathcal{O}\mathcal(B)$ & 7492s & 74.6 \\ 
			\midrule
			IRE$_{\mathrm{RoBERTa}}$ + GNN & $\mathcal{O}(N)$ &  N.A. & N.A. \\ 
			IRE$_{\mathrm{RoBERTa}}$ + \modelname (ours) & $\mathcal{O}(B)$ & 7681s & \textbf{75.5} \\
			\bottomrule
		\end{tabular}
	\end{adjustbox}
	\caption{Training time, the time complexity per training step, and F1 scores of IRE$_{\mathrm{RoBERTa}}$ with our proposed message passing implemented as GNN and \modelname on TACRED. 
	The training time of IRE$_{\mathrm{RoBERTa}}$ with the classical GNN is unavailable due to the our-of-memory error. $B$ and $N$ are the batch and dataset sizes respectively.}\label{tab:time}
\end{table}

We notice that, compared with the classical message passing of GNN, our \modelname method significantly reduces the time complexity per training step.
As a result, our \modelname method takes significantly less training time than the classical GNN method, and exhibits similar efficiency to the original IRE$_{\mathrm{RoBERTa}}$ without message passing between sentences.
The running time and F1 of IRE$_{\mathrm{RoBERTa}}$ with GNN is unavailable due to the out-of-memory error.
This agrees with the theoretical analysis in \Cref{sec:3_2}.
$N$ and $B$ denote the data and batch sizes respectively.
IRE$_{\mathrm{RoBERTa}}$'s time complexity is $\mathcal{O}(B)$, which is the same as the original RoBERTa, while the time complexity of RoBERTa with GNN is $\mathcal{O}(N)$, being significantly higher than our \modelname.
In practice, $N$ is generally large, and $N \gg B$, e.g., $|\mathcal{E}| > 1 \times 10^5$ and $B < 100$ holds for TACRED and state-of-the-art models.

\begin{table}[tb!]
	\centering

	\begin{adjustbox}{width=\linewidth}
		
		\begin{tabular}{@{}l| c c @{}}
			\toprule
			\textbf{Method}
			& \textbf{TACRED}	
			& \textbf{TACREV} \\ 
			\midrule
			\midrule
			LUKE \cite{yamada-etal-2020-luke} & 76.5 & 82.9 \\
			LUKE + \modelname (ours) & 78.9 & 85.6\\
			\midrule
			IRE$_{\mathrm{RoBERTa}}$ \cite{zhou2021improved} & 78.7 & 86.9\\
			IRE$_{\mathrm{RoBERTa}}$ + \modelname (ours) & \textbf{80.1} & \textbf{88.2} \\
			\bottomrule
		\end{tabular}
		
	\end{adjustbox}
	\caption{Test F1 scores (\%) of Relation Extraction on the filtered test sets (see \Cref{sec:unseen}), i.e., the instances containing unseen entities.}\label{tab:unseen}

\end{table}

In terms of effectiveness, our \modelname leads to substantial improvements for RoBERTa. 
Our \modelname enriches the input features for RE on every sentence by utilizing the dataset-level information beyond the individual sentences.
\modelname implements the attention module to incorporate the global property features from different topic caches with adaptive weights, which capture the most relevant information for the target relation.
The improvements in effectiveness are rooted in the message passing mechanism between sentences, which mines the property information beyond individual instances and acts as a complementary to the sentence-level semantics.
Our \modelname method resolves the efficiency issues of message passing based on the caching mechanism, which updates the properties' representations in an online manner.

\begin{table}[!tb]
	\centering
	\vspace{1mm}
	\begin{adjustbox}{width=\linewidth}
		
		\begin{tabular}{@{}l|c|c|c@{}}
			\toprule 
			\textbf{Technique}
			& \textbf{F1 (\%)}
			& \textbf{$\Delta$}
			& \textbf{Cumu $\Delta$}\\ \midrule
			LUKE \cite{yamada-etal-2020-luke} & 72.7 & 0 & 0 \\
			+ Entity Types & 73.4 & +0.7 & +0.7\\
			+ Contextual Topics & 74.8 & \textbf{+1.4} & \textbf{+2.1} \\
			\bottomrule
		\end{tabular}
	\end{adjustbox}
	\caption{Effects of different properties in our heterogeneous graph on the RE of TACRED. \label{tab:components}}
\end{table}

\subsection{Analysis on Unseen Entities} \label{sec:unseen}
Some previous work \cite{zhang2018graph,joshi-etal-2020-spanbert} suggests that RE models may not generalize well to unseen entities.
To evaluate whether the RE models can generalize to unseen entities, existing work designs a filtered evaluation setting \cite{zhou2021improved}.
This setting removes all testing instances containing entities from the training set of TACRED and TACREV, which results in filtered test sets of 4,599 instances on TACRED and TACREV.
These filtered test sets only contain instances with unseen entities during training.

We present the experimental results on the filtered test sets in \Cref{tab:unseen}.
Our \modelname still achieves consistently substantial improvements for LUKE and IRE$_{\mathrm{RoBERTa}}$ on the TACRED and TACREV datasets.
Specifically, our \modelname improves the F1 scores of LUKE by 3.1\% on TACRED, 3.3\% on TACREV, and improves IRE$_{\mathrm{RoBERTa}}$ by 1.8\% on TACRED, 1.5\% on TACREV.
Taking a closer look, we observe that the improvements given by \modelname on the filtered test sets are generally larger than those on the original test sets.
The reason is that our \modelname mines global information from the whole dataset and uses it as the prior knowledge for RE, which is not influenced by the entity names in individual sentences.
When the entity names are new to the RE models, the semantic information is relatively scarce and our mined global information plays a more important role to augment the sentence-level representations.

\begin{figure}[!tb]
	\centering
	\includegraphics[width=0.7\linewidth]{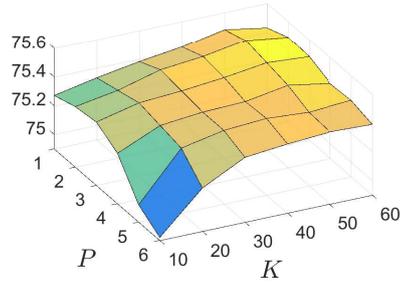}
	\caption{
		The F1 scores (\% in z-axis) of IRE$_{\mathrm{RoBERTa}}$ with \modelname on relation extraction on TACRED with different hyper-parameters $P$ and $K$.\label{fig:hyper}}
\end{figure}

\begin{table*}[!tb]
	\centering
	\renewcommand\arraystretch{2.3}
	\begin{adjustbox}{width=\linewidth}

		\begin{tabular}{@{}lcccc @{}}
			\toprule 
			\textbf{Input sentence}
			& \textbf{Method}
			& \textbf{Prediction}
			& \textbf{Entity type}
			& \textbf{Topic keyword}
			\\ \midrule
			\multirow{2}{6cm}{Founded in \uwave{1947} by two brothers, Eugene and \underline{Quentin Fabris}, New Fabris started out making sewing machine parts in the 1990s.}
			& LUKE & founded \xmark
            & \multirow{2}{3cm}{subject: Person object: Date}
            & \multirow{2}{5cm}{[brother, found, sister, parent, establish, machine, business, organize, instrument, make]} \\

            \cmidrule(l){2-3}
			& + \modelname & \textbf{no\_relation} \cmark &  \\
			\midrule
			\multirow{2}{6cm}{According to the suspect, \underline{Gonzalez} was strangled and buried \uwave{the day} after the video was made, Rosas said.}
			& LUKE & no\_relation \xmark
			& \multirow{2}{3cm}{subject: Person object: Date}
            & \multirow{2}{5cm}{[strangle, die, after, when, injury, day, hospital, police,  murder, later]} \\
			\cmidrule(l){2-3}
			& + \modelname & \textbf{date\_of\_death} \cmark &  \\
			\midrule
			\multirow{2}{6cm}{He was forced to close his bar and now works occasionally at the \uwave{University of Foreigners}, which \underline{Knox} and Kercher attended.}
			& LUKE & no\_relation \xmark
			& \multirow{2}{3cm}{subject: Person object: Organization}
            & \multirow{2}{5cm}{[university, student, attend, opening, work, school, job, professor, exchange, education]} \\
            \cmidrule(l){2-3}
			& + \modelname & \textbf{schools\_attended} \cmark & \\
			\midrule
			\multirow{2}{6cm}{Margaret Garritsen graduated from the \uwave{University of Michigan} as an \underline{American Association of University} scholar.}
			& LUKE & schools\_attended \xmark
			& \multirow{2}{3cm}{subject: Organization object: Organization}
            & \multirow{2}{5cm}{[graduate, government, association, degree, university, technology, science, scholar, receive, research]} \\
            \cmidrule(l){2-3}
			& + \modelname & \textbf{no\_relation} \cmark\\
			\bottomrule
		\end{tabular}
	\end{adjustbox}
	\caption{A case study for LUKE and our \modelname on the relation extraction dataset TACRED.
	We report the predicted relations of different methods, the entity types, and	the top 10 words with the highest probabilities of the topic that the sentence attends with the highest attention weight.}
	\label{tab:case}
\end{table*}

\subsection{Ablation Study}
We investigate the contributions of properties that we consider for constructing the heterogeneous graph.
We apply different kinds of properties sequentially with our \modelname on the LUKE model.
The results are presented in \Cref{tab:components}.
Our entity type nodes improve the effectiveness of LUKE by modeling the entity information globally on the dataset level to enrich the semantics of every sentence.
This finding is consistent with \citet{peng2020learning}, suggesting that the entity information can provide richer information to improve RE.
Furthermore, the contextual topics lead to more significant improvements than the entity types, since the contextual information is fundamental for identifying the relations.

Finally, we analyze the sensitivity of \modelname to the hyper-parameters $K, P$, where $K$ is the number of topics and $P$ is the number of relevant topics assigned to an instance.
The result is visualized in \Cref{fig:hyper}.
We vary $K$ among $\{10, 20, 30, 40, 50, 60\}$ and $P$ among $\{1, 2, 3, 4, 5, 6\}$.
The performance of IRE$_{\mathrm{RoBERTa}}$ with \modelname is relatively smooth when parameters are within certain ranges.
However, extremely small values of $K$ and large $P$ result in poor performances.
Too small $K$ cannot effectively model the complex contextual topics in the large text corpus, while too large $P$ induces irrelevant or noisy features for every instance.
Moreover, only a poorly set hyper-parameter does not lead to significant performance degradation, which demonstrates that our \modelname framework is able to effectively mine the beneficial properties at the dataset level and use them to enhance the relation representations for RE.

\subsection{Case Study}
We conduct a case study to investigate the effects of our \modelname.
\Cref{tab:case} gives a qualitative comparison example between LUKE and the LUKE with our \modelname on the relation extraction dataset TACRED.
The result shows that the global property information that we mine from the whole dataset can guide the RE systems to make correct predictions.
For example, in the first row, we model the global entity type information of the subject as the \textit{person} and the object as the \textit{date} from the whole dataset.
This type information acts as the prior knowledge that prevents the model from making the wrong relation prediction of `founded' between the entities `\textit{Quentin Fabris}' and `\textit{1947}' (date).
Similarly, in the final row, our \modelname filters out the incorrect relation `\textit{schools\_attend}', since
we model the entity type information from the whole dataset and thus enable the model to be aware that this relation cannot hold for the subject type as `\textit{organization}'.

In addition, in the second row, the sentence `\textit{According to the suspect, \underline{Gonzalez} was strangled and buried \uwave{the day} after the video was made, Rosas said.}'  attends to the topic of keywords `[strangle, die, after, when, injury, day, hospital, police, murder, later]' in our heterogeneous graph, which enriches the semantics of the sentence with the context related to the death and time.
This helps the model to make the correct relation prediction '\textit{date\_of\_death}'.

\section{Conclusion}
In this paper, we study the efficient message passing to enhance the relation extraction models.
We propose a novel method named \modelname, which provides efficient message passing between instances in the whole dataset.
\modelname is a model-agnostic technique that can be incorporated into popular relation extraction models to enhance their effectiveness without increasing their time complexity.
In our work, we present a simple yet effective implementation of \modelname, which models two universal and essential properties for relation extraction: entity information and textual context.
Our experimental results show that \modelname, with our heterogeneous graph, yields significant gains for the sentence-level relation extraction in an efficient manner. 

\section*{Acknowledgement}
The authors would like to thank the anonymous reviewers for their discussion and feedback.

Muhao Chen and Wenxuan Zhou are supported by the National Science Foundation of United States Grant IIS 2105329, and by the DARPA MCS program under Contract No. N660011924033 with the United States Office Of Naval Research.
Except for Muhao Chen and Wenxuan Zhou, this paper is supported by NUS ODPRT Grant R252-000-A81-133 and Singapore Ministry of Education Academic Research Fund Tier 3 under MOEs official grant number MOE2017-T3-1-007.

\bibliography{acl2020}
\bibliographystyle{acl_natbib}

\end{document}